\documentclass[12pt]{article}

\usepackage[margin=1in]{geometry}

\usepackage[utf8]{inputenc} 
\usepackage[T1]{fontenc}    
\usepackage{hyperref}       
\usepackage{url}            
\usepackage{booktabs}       
\usepackage{amsmath}
\usepackage{amsfonts}       
\usepackage{nicefrac}       
\usepackage{microtype}      
\usepackage{graphicx}
\usepackage{bibentry}
\usepackage{subcaption}
\usepackage{wrapfig}
\usepackage[colorinlistoftodos]{todonotes}
\usepackage{algorithm}
\usepackage{algpseudocode}
\usepackage{caption}
\usepackage{subcaption}
\usepackage{todonotes}
\usepackage{tcolorbox}
\usepackage{amsthm}
\usepackage{enumitem}
\usepackage[colorinlistoftodos]{todonotes}

\graphicspath{{figures/}}

\theoremstyle{definition}
\newtheorem{proposition}{Proposition}
\newtheorem{lemma}{Lemma}
\newtheorem{definition}{Definition}
\newtheorem{example}{Example}
\newtheorem{remark}{Remark}

\newtheorem{corollary}{Corollary}

\newcommand{\transp}[1]{{#1}^{\top}}
\newcommand{\ud}{\mathrm{d}}

\title{ \bf
Online greedy identification of\\linear dynamical systems
}

\author{Matthieu Blanke
\thanks{INRIA, DI/ENS, PSL Research University, Paris, France.
}%
\and Marc Lelarge%
\footnotemark[1]
}

\begin{document}

\maketitle

\begin{abstract}
    This work addresses the problem of exploration in an unknown environment. For linear dynamical systems, we use an experimental design framework and introduce an online greedy policy where the control maximizes the information of the next step. In a setting with a limited number of experimental trials, our algorithm has low complexity and shows experimentally competitive performances compared to more elaborate gradient-based methods.
    \footnote{Our code is available at \url{https://github.com/MB-29/greedy-identification}}
\end{abstract}

\section{Introduction}
\label{section:introduction}

System identification is a problem of great interest in many fields such as econometrics, robotics, aeronautics, mechanical engineering or reinforcement learning \cite{ljung1998system, NATKE19921069, goodwin1977dynamic, gupta1976application, Moerland2021}. The task consists in estimating the parameters of an unknown system by sampling trajectories from it as fast as possible. To this end, inputs must be chosen so as to yield maximally informative trajectories.
We focus on linear time-invariant~(LTI) systems. Let $A \in \mathbb{R}^{d \times d}$ and~${B\in \mathbb{R}^{d \times m}}$ be two matrices; we consider the following discrete-time dynamics:
\begin{equation}
    \begin{aligned}
        x_0     & =0,
        \\
        x_{t+1} & = A x_t + B u_t + w_t, \quad 0 \leq t \leq T-1
    \end{aligned}
    \label{eq:controlled_dynamics}
\end{equation}
where $x_t \in \mathbb{R}^{d}$ is the state, $w_t \sim \mathcal{N}(0, \sigma^2 I_d)$ is a normally distributed isotropic noise with known variance ~$\sigma^2$ and the control variables~$u_t \in \mathbb{R}^m$ are chosen by the controller with the following power constraint:
\begin{equation}
    \frac{1}{T} \sum\limits_{t=0}^{T-1} \left\lVert u_t \right\rVert^2 \leq \gamma^2.
    \label{eq:power_constraint}
\end{equation}
The system parameters  $(A \, B) := \theta \in \mathbb{R}^{d\times q}$ ($q=d+m$) are unknown initially and are to be estimated from observed trajectories~$(x_t)$. 
The goal of system identification is to choose the inputs $u_t$ so as to drive the system to the most informative states for the estimation of~$\theta$. It may happen that the controller knows~$B$, in which case $\theta = A$ and $q=m$.
\par
System identification is a primary field in control theory. It has been widely studied  in the field of optimal design of experiments \cite{fedorov1972theory,980881}. For LTI dynamic systems, classical optimal design approaches provided results for single-input single-output (SISO) systems \cite{goodwin1977dynamic, 10.2307/1268065} or multi-input multi-output (MIMO) systems in the frequency domain or with randomized time-domain inputs \cite{MEHRA1976211}.
More recently, system identification received considerable attention in the machine learning community, with the aim of obtaining finite-time bounds on the estimation error for~$A$ \cite{jedra2020finite, jedra2020finitetime, Simchowitz2018}. In \cite{Wagenmaker2020} and \cite{wagenmaker2021taskoptimal}, the inputs are optimized in the frequency domain to maximize an optimal design objective, with theoretical estimation rate guarantees.
In our approach, we directly optimize deterministic inputs in the time domain for MIMO LTI systems.
An important aspect of system identification is the quantity of computational resource and the number of observations needed to reach a certain performance.
We study the computational complexity of our algorithms and compare their performance against each other and against an oracle, both on average and on real-life dynamic systems.

\subsection{Notations}
\begin{sloppypar}
    In the rest of this work, we note $\theta_{\star}=(A_{\star} B_{\star})$ the unknown parameter underlying the dynamics. We suppose that the pair $(A_{\star}, B_{\star})$ is controllable: the matrix~${R_{\star}= (B_{\star} \, A_{\star}B_{\star}\,  \dots \, A_{\star}^{d-1}B_{\star})}$ has rank~$d$. Adopting the notations of \cite{wagenmaker2021taskoptimal}, we define a policy $\pi : (x_{1:t},u_{0:t-1})\rightarrow u_t$ as a mapping from the past trajectory to future input. The set of policies meeting the power constraint~\eqref{eq:power_constraint} is noted $\Pi_{\gamma}$. We note ${\tau = (x_{1:T},u_{0:T-1})}$ a trajectory, and we extend this notation to $\tau(\pi, T)$ when the trajectory is obtained using a policy $\pi$ up to time $T$. We denote by $\mathbb{E}_\theta$ the average for a dynamical system given by \eqref{eq:controlled_dynamics} (where the randomness comes from the noise $w_t$ and possibly from the policy inducing the control $u_t$).
\end{sloppypar}

\subsection{Adaptive identification}
Fix an estimator $\hat{\theta}: \tau \mapsto \hat{\theta} (\tau) \in \mathbb{R}^{d\times q}$, yielding an estimate of the parameters from a given trajectory. Our objective is to play the inputs $u_t$ of a policy $\pi \in \Pi_{\gamma}$ so that the resulting trajectory $\tau$ gives a good estimation $\hat{\theta}(\tau)$ for~$\theta_{\star}$.  We measure this performance by the mean squared error:
\begin{equation}
    \label{eq:MSE}
    \mathrm{MSE}(\pi) = \frac{1}{2} \mathbb{E}_{\theta_\star} \left[ \left\lVert \hat{\theta}\big(\tau(\pi, T)\big) - \theta_{\star} \right\rVert_{\mathrm{F}}^2 \right].
\end{equation}
Of course, this quantity depends on $\theta_{\star}$ the true parameter of the system which is unknown. A natural way of proceeding is to estimate $\theta_{\star}$ sequentially, as follows.

\begin{definition}[Adaptive system identification]
    Given an estimate $\hat{\theta}_i$ of $\theta_{\star}$, the policy for the next sequence of inputs can be chosen so as to minimize a cost function $F$ approximating the MSE \eqref{eq:MSE}, using $\hat{\theta}_i$ as an approximation of $\theta_{\star}$. Then, these inputs are played and $\theta_{\star}$ is re-estimated with the resulting trajectory, and so on. We call planning the process of minimizing $F$.
\end{definition}
This approach is summarized in Algorithm \ref{algorithm:sequential-identification}, which takes as inputs a first guess for the parameters to estimate $\theta_0$ and a policy $\pi_0$, the problem parameters $\sigma$ and $\gamma$, a schedule~${\{t_0, =0, t_1, \dots, t_{n-1}, t_n=T \}}$, a cost functional~$F$ and an estimator $\hat{\theta}$.
\begin{algorithm}
    \caption{Sequential system identification}
    \label{algorithm:sequential-identification}
    \begin{algorithmic}
        \State \textbf{inputs} initial guess $\theta_0$, $\pi_0$, noise variance $\sigma^2$, power $\gamma^2$, cost functional $F$, estimator $\hat{\theta}$
        \State \textbf{output} final estimate $\theta_T$
        \For{$0 \leq i \leq n-1$}
        \State  run the true system $t_{i+1} - t_i$ steps
        \State \; with inputs $u_t = \pi_i(x_{1:t}, u_{1:t-1})$
        \State $\theta_i  = \hat{\theta}(x_{1:t_i}, u_{1:t_i-1})$ \Comment{estimation}
        \State  $\pi_i$ solves $\underset{\pi \in \Pi_{\gamma}}{\min} \, F(\pi ; \theta_i, t_{i+1})$ \Comment{planning}
        \EndFor
    \end{algorithmic}
\end{algorithm}
An adaptive identification algorithm is hence determined by a triplet $(\hat{\theta}, F, \{ t_i \})$.
A natural estimator is the least squares estimator $\hat{\theta} = \hat{\theta}_{\mathrm{LS}}$ which we define in Section~\ref{section:OLS}. In the rest of this work, we set $\hat{\theta} = \hat{\theta}_{\mathrm{LS}}$. 
\begin{example}[Random policy]
    \label{example:random}
    A naive strategy for system identification consists in playing random inputs with maximal energy at each time step. This corresponds to the choice~$t_i = i$ and $\pi_i$ returning $u_t \sim \mathcal{N}(0, \gamma^2/m)$.
\end{example}
\begin{example}[Task-optimal pure exploration]
    \label{exemple:tople}
    In \cite{wagenmaker2021taskoptimal}, the authors propose the following cost function
    \begin{equation}
        \label{eq:tople_cost} 
        F(\pi ; \theta, t) = \mathrm{tr} \left[ \left( \Gamma_t\bigl( \tau(\pi) ; \theta \bigr) \right)^{-1} \right],
    \end{equation}
    where $\Gamma_t$ is defined in equation \eqref{eq:gramians} below.
    As we will see in Section~\ref{section:OD}, this corresponds to A-optimal experimental design. The authors show that this cost function approximates the MSE in the long time limit at an optimal rate when $T \rightarrow + \infty$. In their identification algorithm, they set $t_i = 2^i \times T_0$ for some initial epoch $T_0$.
\end{example}

\begin{example}[Oracle]
    \label{example:oracle} An oracle is a controller who is assumed to choose their policy with the knowledge of the true parameter $\theta_{\star}$. It can hence perform one single, offline optimization of $F(\pi; \theta, T) = \mathrm{MSE}(\pi)$ over~$\{ t_i \} = \{0, T \}$. By definition, the inputs played by the oracle are the optimal inputs for our problem of mean squared error system identification.
\end{example}

\subsection{Contributions}

In practice, systems have complex dynamics and can only be approximated locally by linear systems. We still believe that in order to understand complex systems, we need to understand identification of linear systems as on short time scales, we can approximate the complex system with a linear one. In order to be practical, our identification algorithm needs to interact as little as possible with the true system and to take decisions as fast as possible. With previous notations, we are interested in cases where $T$ is small (to ensure that in practice the dynamics remains time-invariant and linear) and where the estimation and planning steps need to be very fast in order to run the algorithm online.

In this work, we explore a setting for linear system identification with hard constraints on the number of interactions with the real system and on the computing resources used for planning and estimation. To the best of our knowledge, finite-time system identification guarantees are only available in the large $T$ limit which makes the hypothesis of linear dynamic quite unlikely. Using a framework based on experimental design, we propose a greedy online algorithm requiring minimal computing resources. The resulting policy gives a control that maximizes the amount of information collected at the next step. We show empirically that for short interactions with the system, this simple approach can actually outperforms more sophisticated gradient-based methods.
 We also propose a method to compute an oracle optimal control, against which we can compare the different identification algorithms.

\subsection{Related work}
System identification has been studied extensively in the last decades \cite{gevers:hal-00765659,ljung1998system}.
The question of choosing the maximally informative input can be tackled in the framework of classical experimental design \cite{fedorov1972theory, pukelsheim2006optimal}. Several methods have been proposed for the particular case of dynamic systems \cite{1100554,10.2307/1268065, MEHRA1976211} A comprehensive study can be found in \cite{goodwin1977dynamic}, with a focus on SISO systems.
\par
In the machine learning community, the last few years have seen an increasing interest in finite-time system identification \cite{Simchowitz2018,jedra2020finitetime, sarkar2019near,tsiamis2019finite}. These works typically derive theoretical error rates for linear dynamic system estimation and produce high probability bounds guaranteeing that the estimation is smaller than~$\varepsilon$~with probability greater than~$1- \delta$ after a certain number of samples. The question of designing optimal inputs is tackled in \cite{Wagenmaker2020,wagenmaker2021taskoptimal}. The authors derive an asymptotically optimal algorithm by computing the control in the frequency domain. In \cite{Mania2020}, an approach to control partially nonlinear systems is proposed.

\section{Background}

It is convenient to describe the structure of the state as a function of the inputs and the noise. By integrating the dynamics \eqref{eq:controlled_dynamics}, we obtain the following result.
\begin{proposition}
    \label{proposition:linear}
    The state can be expressed as
    $x_t = \bar{x}_t + \tilde{x}_t$ with
    \begin{equation}
        \label{eq:x_linear}
        \bar{x}_t =
        {\sum\limits_{s=0}^{t-1}A^{t-1-s}Bu_s},
        \quad
        \tilde{x}_t= {\sum\limits_{s=0}^{t-1}A^{t-1-s}w_s}.
    \end{equation}
\end{proposition}
\par \medskip
Note that that~$\bar{x}_t = \mathbb{E}_\theta[x_t]$ solves the deterministic dynamics~$\bar{x}_{t+1} = A\bar{x}_t + B u_t$
and $\tilde{x}_t$ has zero mean and is independent of the control. The two terms $\bar{x}_t$ and $\tilde{x}_t$ depend linearly on the $Bu_s$ and the $w_s$ respectively.
\par \medskip
The data-generating distribution knowing the parameter~$\theta$ can be computed using the probability chain rule with the dynamics \eqref{eq:controlled_dynamics}:

\begin{equation}
    p(\tau|\theta) = \frac{1}{\sqrt{2 \pi \sigma^2}}\exp \left[  - \frac{1}{2\sigma^2} \sum\limits_{t=0}^{T-1}
        \left\lVert x_{t+1} - A x_t - B u_t \right\rVert^2_2 \right].
\end{equation}
We define the log-likelihood (up to a constant):
\begin{equation}
    \label{eq:log-likelihood}
    \begin{aligned}
        \ell(\tau, \theta) & = - \frac{1}{2\sigma^2} \sum\limits_{t=0}^{T-1}
        \left\lVert x_{t+1} - A x_t - B u_t \right\rVert^2_2
        \\
                           & = -\frac{1}{2 \sigma^2} \Vert Y - Z \transp{\theta} \Vert^2_{\mathrm{F}},
    \end{aligned}
\end{equation}
where we have noted
${Y = \transp{(y_0 \, \dots \, y_{T-1})}} \in \mathbb{R}^{T \times d}$
and
${Z = \transp{(z_0 \, \dots \, z_{T-1}}) \in \mathbb{R}^{T \times q}}$
the observations and the covariates associated to the parameter $\theta$.
If $\theta = (A \, B)$, then $y_t = x_{t+1}$,
$z_t =
    \begin{pmatrix}
        x_t \\ u_t
    \end{pmatrix}$.
If $\theta = A$, then $y_t = x_{t+1} - B u_t$ and~$z_t = x_t$. We also note ${U = \transp{(u_0 \, \dots \, u_{T-1}}) \in \mathbb{R}^{T \times m}}$ the input matrix and ${X = \transp{(x_0 \, \dots \, x_{T-1}}) \in \mathbb{R}^{T \times d}}$ the state matrix.
We define the moment matrix~$M_t = \sum\limits_{s=0}^{t} z_t \transp{z_t}$
and the Gramians of the system at time $t$:
\begin{equation}
    \label{eq:gramians}
    \Gamma_t( \tau ; \theta ) = \frac{1}{t}\mathbb{E}_{\theta}\left[M_{t-1}
        \right]
\end{equation}
and $G_t(A) = \sum\limits_{s=0}^{t-1} A^s \transp{A^s}$.
Note that $ \transp{Z}Z=M_{T}$


\subsection{Ordinary least squares}
\label{section:OLS}

Given a trajectory, a natural estimator for the matrix $A_{\star}$ is the least squares estimator. The theory of least squares provides us with a formula for the mean squared error with respect to the ground truth, which can be used as a measure of the quality of a control.
\begin{proposition}[Ordinary least squares estimator]
    \label{proposition:OLS}
    \begin{sloppypar}
        Given inputs $U$ and noise $W$, the ordinary least squares~(OLS) estimator associated to the resulting trajectory $X$ is
    \end{sloppypar}
    \begin{equation}
        {\hat{\theta}(\tau) = \big((\transp{Z}Z)^{-1}\transp{Z}Y}\big)\transp{}.
        \label{eq:OLS_estimator}
    \end{equation}
    and its difference to $\theta_{\star}$ is given by

    \begin{equation}
        \begin{aligned}
            \transp{\big(\hat{\theta}(\tau) - \theta_{\star} \big)} & = (\transp{Z}Z)^{-1}\transp{Z}W
            \\ &= Z^+W,
            \label{eq:OLS_difference}
        \end{aligned}
    \end{equation}
    where $Z^+$ denotes the pseudo-inverse of $Z$.
    Noting $\theta_t$ the least squares estimator obtained from the trajectory up to time $t$, we recall the recursive update formula
    \begin{equation}
        \label{eq:OLS_update}
        \transp{{\theta}_{t+1}} = M_{t+1}^{-1}\big(M_t {\theta_t} + z_t \transp{y_t}\big).
    \end{equation}
\end{proposition}
\begin{proof}
    The least squares estimator minimizes the quadratic loss $\frac{1}{2} \sum\limits_{t=0}^{T-1}\Vert x_{t+1} - Ax_t - Bu_t\Vert_2^2$, which writes
    \begin{equation}
        \label{eq:matrix_ls}
        \frac{1}{2}\left\lVert Y - Z \transp{\theta} \right\rVert^2_{\mathrm{F}} = \frac{1}{2}\sum\limits_{j=1}^d \left\lVert Y_j - Z \theta_j \right\rVert^2_2
    \end{equation}
    with $Y_j$ the $j$-th column of $Y$ and $\theta_j$ the $j$-th row of~$\theta$. The $d$ terms of the sum can be minimized independently, with each $\theta_j$ minimizing the least squares of the vectorial relation~$Y_j = Z \beta$. The solution for $\theta_j$ is equal to $\hat{\theta}_j = (\transp{Z}Z)^{-1}\transp{Z}Y_j$ (see \textit{e.g.} \cite{boyd2018introduction}). By concatenating the columns, we obtain that ${\hat{\theta}\transp{} = (\transp{Z}Z)^{-1}\transp{Z}Y}$, which proves~\eqref{eq:OLS_estimator}. Substituting~${Y = Z \transp{\theta_{\star}} + W}$ yields \eqref{eq:OLS_difference}.  Note here that the controllability assumption on $(A_{\star}, B_{\star})$ ensures that $Z$ can be made full rank, and hence that the moment matrix $\transp{Z}Z$ is invertible.
\end{proof}
\begin{definition}[OLS mean squared error]
    \label{definition:ols_mse}
    For a given trajectory $\tau$ generated with a matrix $A_{\star}$ and noise $W$, the Euclidean mean squared error (MSE) is
    \begin{equation}
        \label{eq:ols_error}
        \begin{aligned}
            \Vert \hat{\theta}_{\mathrm{LS}} - \theta_{\star} \Vert_{\mathrm{F}}^2 & = \big\lVert \big((\transp{Z}Z)^{-1}\transp{Z}W\big)\transp{}\big\rVert_2^2
            \\
                                                                                   & = \mathrm{tr} \left[
                Z (\transp{Z}Z)^{-2}\transp{Z}W \transp{W}
                \right].
        \end{aligned}
    \end{equation}
\end{definition}
If the noise $W$ and the covariates $Z$ were independent, then the expected error would reduce to the A-optimal design objective $\mathbb{E}[\mathrm{tr}(\transp{Z}Z)^{-1}]$. It is not the case in our framework since $Z$ is generated with $W$.

\subsection{Classical optimal design}
\label{section:OD}
The correlation between $Z$ and $W$ makes the derivation of a tractable expression for the expectation of \eqref{eq:ols_error} complicated. In this section, we show how a more tractable objective can be computed by applying theory of optimal experimental design \cite{fedorov1972theory,steinberg1984experimental}. In the classical theory of optimal design, the informativeness of an experiment is measured by the size of the expected Fisher information.
\begin{definition}[Fisher information matrix]
    Let $\ell(\tau, \theta) = \log p(\tau|\theta)$ denote the log-likelihood of the data-generating distribution knowing the parameter $\theta$. The Fisher information matrix is defined as
    \begin{equation}
        I(\theta)= -\mathbb{E}_{\theta}\left[\frac{\partial^2 \ell (\tau, \theta)}{\partial \theta^2 }\right] \quad \in \mathbb{R}^{qd \times qd}.
    \end{equation}
\end{definition}
\begin{proposition}
    For the LTI system \eqref{eq:controlled_dynamics},
    \begin{equation}
        I(\theta) = \frac{T}{\sigma^2}\mathrm{diag}(\Gamma_T, \dots, \Gamma_T),
    \end{equation}
    the number of blocks being $d$. Furthermore, $\Gamma_T$ can be computed as
    \begin{equation}
        \Gamma_T = \frac{1}{T} \sum\limits_{t=0}^{T-1} \bar{z}_t \transp{\bar{z}_t} + \sigma^2 G_t(A).
    \end{equation}
\end{proposition}
\begin{proof}
    The log-likelihood \eqref{eq:log-likelihood} can be separated into a sum over the $\theta_j$ as in \eqref{eq:matrix_ls}.
    The quadratic term in $\theta_j$ is $\Vert Z \theta_j \Vert_2^2 = \transp{\theta_j}\transp{Z}Z\theta_j$ and the other terms are constant or linear. Differentiating twice and taking the expectation gives $\mathbb{E}_\theta[\transp{Z}Z]$, which yields the desired result after dividing by~$-\sigma^2$.
    Following the decomposition \eqref{eq:x_linear},
    $z_t \transp{z_t} =     \bar{z}_t \transp{\bar{z}_t} + \tilde{z}_t \transp{\tilde{z}_t}
        +  \bar{z}_t \transp{\tilde{z}_t} + \tilde{z}_t \transp{\bar{z}_t}.$
    Taking the expectation, we obtain
    $\mathbb{E}[z_t \transp{z_t}] = \bar{z}_t \transp{\bar{z}_t} +  \sigma^2 G_t(A)$. Summing over $t$ yields the result.
    Note that the first term is deterministic and depends on the control whereas the second term depends on the noise and not on the control. Therefore, the expected moment matrix is the sum of a noise term and of a deterministic control term.
\end{proof}
\begin{definition}
    In classical optimal design, the size of the information matrix is measured by some criterion~${\Phi : \mathbb{S}_n^+(\mathbb{R}) \rightarrow \mathbb{R}_+}$, which is a functional of its eigenvalues $\lambda_1, \dots, \lambda_d \geq 0$. The quantity $\Phi(I)$ represents the amount of information brought by the experiment and should be maximized.
\end{definition}
\begin{example}
    Some of the usual criteria are presented in Table~\ref{table:criteria}.
\end{example}
The criteria are required to have properties such as homogeneity, monotonicity and concavity in the sense of the Loewner ordering, which can be interpreted in terms of information theory: monotonicity means that a larger information matrix brings a  greater amount of information, concavity means that information cannot be increased by interpolation between experiments. We refer to \cite{pukelsheim2006optimal} for more details.

\begin{table}[!t]
    \renewcommand{\arraystretch}{1.3}
    \caption{Alphabetical design criteria.}
    \label{table:criteria}
    \centering
    \begin{tabular}{|c|l|}
        \hline
        Optimality   & $\Phi(\lambda_1, \dots, \lambda_d)$

        \\
        \hline
        A-optimality & $-\big(1/{\lambda_1} + \dots + 1/{\lambda_d} \big)$

        \\
        D-optimality & $ \log \lambda_1 + \dots \log \lambda_d
        $
        \\
        E-optimality & $\lambda_1$                                         \\
        \hline
    \end{tabular}
\end{table}
The theory of classical optimal design leads to the definition of the following optimal design informativeness functional.
\begin{definition}[Optimal design functional]
    Let $\Phi$ denote an optimal design criterion. Then the associated cost is defined as
    \begin{equation}
        \label{eq:OD-functional}
        \begin{aligned}
            F_\Phi(\pi;\theta, t) & = -\Phi \left[ \Gamma_t\bigl( \tau(\pi) ; \theta \bigr) \right]
            \\
                                  & = -\Phi \left[\sum\limits_{s=0}^{t-1} \bar{z}_s \transp{\bar{z}_s} + \sigma^2 G_s(A)\right],
        \end{aligned}
    \end{equation}
    where the $\bar{z}_s$ depend on the inputs $u_s$ through \eqref{eq:x_linear}.
\end{definition}
\begin{remark}
    \label{remark:quadratic}
    We note from equation \eqref{eq:x_linear} that $Z$ is affine in~$U$. Hence, $\transp{Z}Z$ is quadratic in~$U$, and maximizing \eqref{eq:OD-functional} efficiently is challenging even with concavity assumptions on~$\Phi$.
\end{remark}

\subsection{Small noise regime}
\label{section:low-noise}
The optimal design functional \eqref{eq:OD-functional} can be related to the MSE in the small noise regime~${\sigma \ll \gamma}$.
\begin{proposition}
    \label{proposition:ols_od}
    The A-optimal design functional~\eqref{eq:OD-functional} is a~$\mathcal{O}(\sigma/\gamma)$ approximation of the~MSE~\eqref{eq:MSE}:
    \begin{equation}
        \label{eq:ols_od}
        {\mathrm{MSE}}(\pi) = \frac{1}{2}F_{\mathrm{A}}(\pi; \theta_\star, T) + \mathcal{O}(\sigma/\gamma).
    \end{equation}
\end{proposition}

\begin{proof}
    We introduce the rescaled variables ${\zeta= (1/\gamma) Z}$ and $\omega = ({1}/{\sigma})W$ which are of order~$1$.
    Extending the notations of equation \eqref{eq:x_linear},
    $Z  = \bar{Z} + \tilde{Z}$, where the first term is of order $\gamma$ and the second is of order $\sigma$. Therefore,
    $Z = \bar{Z} + \mathcal{O}(\sigma)$,
    or equivalently
    $\zeta = \bar{\zeta} + \mathcal{O}(\sigma/\gamma)$.
    By Proposition~\ref{proposition:pinverse_differential},~$\zeta^+$ is differentiable at $\bar{\zeta}$ so
    $\zeta^+ =  \bar{\zeta}^+ + \mathcal{O}(\sigma/ \gamma)$.
    Taking the squared norm and using Cauchy-Schwartz inequality, we obtain
    \begin{equation}
        \label{eq:sq-error_expansion}
        \begin{aligned}
            \left\lVert \zeta^+\omega \right\rVert^2 & =\left\lVert \bar{\zeta}^+\omega \right\rVert^2 + \mathcal{O}(\sigma/\gamma).
        \end{aligned}
    \end{equation}
    Furthermore,
    \begin{equation}
        \label{eq:first-term}
        \begin{aligned}
            \mathbb{E} \left[ \left\lVert \bar{\zeta}^+ \omega \right\rVert^2 \right]
             & = \mathbb{E} \left[\mathrm{tr} \left(\bar{\zeta}(\transp{\bar{\zeta}}\bar{\zeta})^{-2} \transp{\bar{\zeta}}\omega \transp{\omega}\right)\right]
            \\
             & = \mathrm{tr}\left[(\transp{\bar{\zeta}}\bar{\zeta})^{-1} \right].
        \end{aligned}
    \end{equation}
    Gathering \eqref{eq:sq-error_expansion} and \eqref{eq:first-term}, we obtain
    \begin{equation}
        \frac{1}{2} \mathbb{E}\left[\left\lVert \zeta^+\omega \right\rVert^2\right] = \frac{1}{2}  \mathrm{tr}\left[(\transp{\bar{\zeta}}\bar{\zeta})^{-1} \right]  + \mathcal{O}(\sigma/\gamma).
    \end{equation}
\end{proof}

\begin{remark}
    In classical least squares regression, the covariates~$Z$ are independent of the noise~$W$. As a consequence, the minimziation of the mean squared estimation error leads to the classical A-optimality criterion. This does not hold in general in our framework because  the signal and the noise are coupled by the dynamics \eqref{eq:controlled_dynamics}.  However, Proposition~\ref{eq:ols_od} shows that this criterion does hold in the small noise regime at first order in $\sigma / \gamma$. Indeed, when~$\sigma \ll \gamma$ the contribution of the noise to the signal is negligible because the deterministic part of the signal is of order $\gamma$.
\end{remark}

\begin{remark}
    \label{remark:scaling}
    From Proposition \ref{proposition:ols_od} and the definition of A-optimality, we see that the MSE approximately scales like~$1/T$ when the number of observations increases. This is confirmed by experiments.
\end{remark}

\section{Online greedy identification}
\label{section:greedy}

\subsection{One-step-ahead objective}
A simple, natural approach for system identification consists in choosing a decision sequentially at each time step. At each time $t$, the control $u_t$ is chosen with energy $\gamma^2$ so as to maximize a one-step-ahead objective. Then, a new observation $x_t$ is collected and the process repeats.
Following~Section~\ref{section:OD}, $u_t$ can be chosen to maximize the value of $F_{\Phi}$ at $t+1$. This corresponds to the choice of functional $F = F_{\Phi}$ and to the one-step schedule $t_{i} = i$.
\par
Upon choosing $u_t$, the policy $\pi_t$ should select $u_t$ so as to maximize the design criterion~$\Phi$ applied on the one-step ahead, $u_t$-dependent information matrix, the past trajectory $x_{0:t}$ being fixed. The one-step-ahead information matrix is~${M_{s-1} + \mathbb{E}_{A_s}[z_{s}\transp{z_{s}}]}$, with $s=t$ when~$B_\star$ is estimated (because then then next $u_t$-dependent covariate is~$z_t$) and~${s = t+1}$ if $B_\star$ is known, because then the next $u_t$-dependent covariate is $x_{t+1}$.
Therefore, one-step ahead planning yields the following optimization problem:
\begin{equation}
    \label{problem:online_deterministic_OD}
    \begin{aligned}
        \underset{u \in \mathbb{R}^m}{\max}
        \qquad & \Phi\left(\bar{M}_t + z(u) \transp{z(u)} \right)
        \\
        \text{such that}
        \qquad &
        \left\lVert u \right\rVert^2 \leq \gamma^2,
    \end{aligned}
\end{equation}
with
\begin{equation}
    \bar{M}_t = \begin{cases}
            M_{t-1} + \sigma^2 G_{t}(A_t) \quad \text{if} \quad \theta = (A, \, B)
            \\ M_{t} + \sigma^2 G_{t+1}(A_t) \quad \text{if} \quad \theta = A,
        \end{cases}
\end{equation}
    and 
\begin{equation}
        z(u) = 
        \begin{cases}
            \begin{pmatrix}
                x_t \\ u
            \end{pmatrix} \quad \text{if} \quad \theta = (A, \, B)
            \\ 
            A_t x_t + B_\star u_t \quad \text{if} \quad \theta = A.
        \end{cases}
\end{equation}
\begin{remark}
    With this greedy policy, the energy constraint imposed for one input ensures that the global power constraint \eqref{eq:power_constraint} is met.
\end{remark}
The corresponding identification process is detailed in Algorithm~\ref{algorithm:greedy_identification}.
We will see in Section~\ref{section:online_optimization} that problem \eqref{problem:online_deterministic_OD} can be solved accurately and at a cheap cost. Moreover, Algorithm~   \ref{algorithm:greedy_identification} offers the advantage of improving the knowledge of~$\theta_{\star}$ at each time step using all the available information on the parameter to plan at each time step. This way, the bias affecting planning due to the uncertainty about $\theta_\star$ is minimized. When planning is performed over larger time sequences, a large bias could impair the identification of the system. 
\begin{algorithm}
    \caption{Greedy system identification}
    \label{algorithm:greedy_identification}
    \begin{algorithmic}
        \State \textbf{inputs} initial guess $\theta_0$, noise variance $\sigma^2$, power $\gamma^2$, time horizon $T$, design criterion $\Phi$
        \State \textbf{output} final estimate $\theta_T$
        \For{$0 \leq t \leq T-1$}
        \State  $u_t \in \underset{\Vert u \Vert^2 \leq \gamma^2}{\mathrm{argmax}} \, \Phi\big( \bar{M}_{t}  +z(u)\transp{z(u)} \big)$

        \State  play $u_t$, observe $x_{t+1}$
        \State $M_{t+1} = M_{t} + x_{t+1}\transp{x_{t+1}}$
        \State $\transp{{\theta}_{t+1}} = M_{t+1}^{-1}\big(M_t {\theta_t} + x_t \transp{y_t}\big)$
        \EndFor
    \end{algorithmic}
\end{algorithm}

\subsection{Solving the one-step optimal design problem }
\label{section:online_optimization}
We show that the one-step ahead planning for online system identification is equivalent to a convex quadratic program which can be solved efficiently.
\par \medskip

\begin{proposition}
    \label{proposition:quadratic_program}
    For D-optimality and A-optimality, there exists a symmetric matrix ${Q \in \mathbb{R}^{m\times m}}$ and $b \in \mathbb{R}^m$   the problem \eqref{problem:online_deterministic_OD} is equivalent to
    \begin{equation}
        \label{problem:quadratic_program}
        \begin{aligned}
            \underset{u \in \mathbb{R}^d}{\min} & \quad \transp{u}Q u -  2\transp{b}u
            \\
            \text{such that}                    & \quad
            \left\lVert u \right\rVert^2 \leq \gamma^2.
        \end{aligned}
    \end{equation}
\end{proposition}
\begin{proof}
    From Proposition \ref{proposition:det_expansion}, we find that
    \begin{equation}
        \begin{aligned}
            \log \det \big(\bar{M_t} + z(u) \transp{z(u)}\big) & = \log \det \bar{M_t}
            \\ &+ \log \big(1 + \transp{z(u)}{\bar{M_t}}^{-1}z(u) \big).
        \end{aligned}
    \end{equation}
    Similarly, from Corollary \ref{corollary:trace_formula} ,
    \begin{equation}
        \begin{aligned}
            -\mathrm{tr} \left[ \left(\bar{M_t} + z(u) \transp{z(u)}\right)^{-1}\right] & = 1-\mathrm{tr} \left[{\bar{M_t}}^{-1}\right]
            \\ &- \frac{1}{1 + \transp{z(u)}{\bar{M_t}}^{-1}z(u)}.
        \end{aligned}
    \end{equation}
    Maximizing these quantities with respect to $u$ amounts to maximizing $\transp{z(u)}{\bar{M_t}}^{-1}z(u)$. The matrix  ${\bar{M_t}}^{-1}$ is symmetric because 
    the~$M_t$ and the~$G_t$ are symmetric, and so are its diagonal submatrices.  Given the affine dependence of $z$ in $u$ and the (possible) block structure of $z$ and~$M_t$,~$\transp{z(u)}{\bar{M_t}}^{-1}z(u)$ is of the form $\transp{u}Q u - 2 \transp{b}u$, up to a constant. We provide an explicit formula for $Q$ and $b$ in the case where $\theta = A$ in Remark~\ref{remark:known-B}
\end{proof}
We now characterize the minimizers of Problem \eqref{problem:quadratic_program}. If a minimizer can be found in the interior of the constraining sphere, then $Q$ is positive semidefinite and the problem can be tackled using unconstrainted optimization. We thus consider the equality constrained problem
\begin{equation}
    \label{problem:equality-constrained}
    \begin{aligned}
        \underset{u \in \mathbb{R}^d}{\min} & \quad \transp{u}Q u -  2\transp{b}u
        \\
        \text{such that}                    & \quad
        \left\lVert u \right\rVert^2 = \gamma^2.
    \end{aligned}
\end{equation} 

\begin{proposition}
    Note $\{\alpha_i \}$ the eigenvalues of~$Q$, and~$u_i$ and~$b_i$ the coordinates of~$u_*$ and~$b$ in a corresponding orthonormal basis. Then a minimizer $u_*$ satisfies the following equations for some nonzero scalar $\mu$:
    \begin{equation}
        u_i = b_i/(\alpha_i + \mu)
        \quad \text{and} \quad
        \sum_i \frac{{b_i}^2}{(\alpha_i + \mu)^2} = \gamma^2
        .
    \end{equation}
\end{proposition}
\begin{proof}
    By the Lagrange multiplier theorem there exists a nonzero scalar~$\mu$ such that~${Qu_* - b = -\mu \, u_*}$, where $\mu$ can be scaled such that $Q+\mu I_d$ is nonsingular. Inverting the optimal condition and expanding the equality constraint gives the two conditions. 
    \end{proof}
Problem \eqref{problem:quadratic_program} can hence be solved at the cost of a scalar root-finding and an eigenvalue decomposition.
In \cite{hager2001minimizing}, bounds are provided so as to initialize the root-finding method efficiently.

\begin{remark}
    \label{remark:known-B}
    In the case where $B_\star$ (\textit{i.e.} $\theta = A$), $Q$ and $b$ have the following expressions:
        \begin{equation}
                Q = -\transp{B}{\bar{M_t}}^{-1}B, \quad b = \transp{B}\bar{M_t}^{-1}A_t x_t.
        \end{equation}
\end{remark}

\section{Gradient-based identification}
\label{section:performance}
In this section, we propose a gradient-based approach to planning.
In a sequential identification scheme of Algorithm~\ref{algorithm:sequential-identification}, the cost functions \eqref{eq:MSE} and \eqref{eq:OD-functional} can be optimized by projected gradient descent. This builds on the following remark.
\begin{remark}[Differentiability of the functionals]
    The functionals \eqref{eq:MSE} and \eqref{eq:OD-functional} are differentiable functions of the output. Indeed, $X$ is an affine function of the inputs as shown in Proposition~\ref{proposition:linear}, and the controllability of $(A, B)$ guarantees that $\transp{Z}Z$ is positive definite.
    Furthermore, the operations of pseudo-inverse (see Proposition \ref{proposition:pinverse_differential}) and the optimal design criteria of Table~\ref{table:criteria} are differentiable over the set of positive definite matrices.
\end{remark}
The gradients with respect to $U$ can either be derived analytically (see \cite{goodwin1977dynamic}, section~6 for the derivation of an adjoint equation) or automatically in an automatic differentiation framework. We rescale $U$ at each step to ensure the power constraint is met. The $t_i$ are chosen arbitrarily. The computational complexity of the algorithm is linear in~$T$: each gradient step backpropagates through the planning time interval.
\subsection{Gradient-based optimal design}
We propose a gradient-based method to optimize $U$ by performing gradient descent directly on $U$ in functional \eqref{eq:OD-functional}. Note that we optimize the inputs directly in the time domain, whereas other approaches such as \cite{wagenmaker2021taskoptimal} perform optimization in the frequency domain by restricting the control to periodic inputs.

\subsection{Gradient through the oracle MSE}
Given the true parameters $\theta_{\star} = (A_{\star} \, B_{\star})$, the optimal control for the MSE minimizes the~MSE cost \eqref{eq:MSE}, as explained Example~\ref{example:oracle}. However, the dependency between $Z$ and $W$ makes this functional complicated to evaluate and to minimize with respect to the inputs, even when the true parameters $\theta_\star$ are known. We propose a numerical method to minimize \eqref{eq:MSE} using automatic differentiation an Monte-Carlo sampling. Given one realization of the noise and inputs $U$, the gradient of the squared error \eqref{eq:ols_error} can be computed automatically in an automatic differentiation framework. Then, one can sample a batch of $b$ noise matrices~${W_1, \dots, W_b \sim \mathcal{N}(0, \sigma^2 I)}$ and approximate the gradient of~\eqref{eq:MSE} by
\begin{equation}
    \nabla \mathrm{MSE}(U)\simeq \frac{1}{b} \sum\limits_{i=1}^b \nabla_U \mathrm{tr} \left[
        Z (\transp{Z}Z)^{-2}\transp{Z}W_i \transp{W_i}
        \right].
\end{equation}
Although we do not have convergence guarantees due to the lack of structure of the objective function, the gradient descent does converge in practice, to a control that outperforms the adaptive controls.

\begin{algorithm}
    \caption{Planning by projected gradient descent}
    \label{algorithm:gradient_planning}
    \begin{algorithmic}
        \State \textbf{inputs} $A_t$, $\sigma$, $\gamma$, $T$, $\eta$, $H_t$
        \State \textbf{output} control $U \in \mathbb{R}^{(T-t)\times m}$
        \For{$0 \leq j \leq n_{\mathrm{gradient}}$}

        \State  $G(U) = F[X(U)|H_t]$
        \State  $U = U - \eta \nabla G(U)$
        \State  $U = (\gamma\sqrt{T} / \Vert U \Vert_{\mathrm{F}}) \times U $
        \EndFor
    \end{algorithmic}
\end{algorithm}


\section{Performance study}
\label{section:experiments}

\subsection{Complexity analysis}
\begin{definition}[Performance]
    Let $\theta_T$ denote the estimation produced by the learning algorithm at the end of identification. The performance of the policy $\pi$ is measured by the average error over the experiments on the true system:~${\varepsilon = \mathrm{MSE}(\pi)}$.
    We study the performance of our algorithms as a function of the number of observations~$T$ and $C$ the computational cost. We also introduce the computational rate~$c = C/T$.
\end{definition}
Algorithm \ref{algorithm:greedy_identification} and the gradient identification algorithm have linear time complexity. Hence, we define $c_{\mathrm{greedy}}$ and $c_{\mathrm{gradient}}$ for a given number of gradient iterations. In practice, we find that~$c_{\mathrm{greedy}} \ll c_{\mathrm{gradient}}$, where $c_{\mathrm{gradient}}$ is the computational rate needed for the gradient descent to converge.
As pointed out in Remark~\ref{remark:scaling}, the squared error essentially scales like~$1/T$. This is verified experimentally.
Given the previous observations, we postulate that the performance of our algorithms takes the form
\begin{equation}
    \label{eq:performance}
    \varepsilon(C, T) = {\eta(c)}/{T}.
\end{equation}
We build an experimental diagram where we plot the average estimation error for $\theta_\star = A_\star$ as a function of the two types of resource $T$ and $C$ for the gradient algorithm.  Increasing~$C$ allows for more gradient steps. We run trials with random matrices $A_{\star}$ of size $d=4$, with $B=I_d$.  We set $\gamma = 1$, $\sigma = 10^{-2}$, $T\in [60, 220]$. The gradient algorithm optimizes the~A-optimality functional \eqref{eq:OD-functional} with a batch size of $b=100$ and~$\{t_i\} = \{0, 10, T/2, T\}$. The obtained performances are to be compared with those of the greedy algorithm (with the A-optimality cost function), which has a fixed, small computational rate $c$. Our diagrams are plotted on~Fig.~\ref{fig:diagrams-exp}.
\par
Our diagrams show that the greedy algorithm is preferable in a phase of low computational rate: $C < c\times T$, as suggested by \eqref{eq:performance}. The phase separation corresponds to a relatively high number of gradient steps. Indeed, the iso-performance along this line are almost vertical, meaning that the gradient descent has almost converged. Furthermore, the maximum performance gain of the gradient algorithm relatively to the greedy algorithm is of 10\%.
\begin{figure}
    \centering
    \includegraphics[width=.7\textwidth]{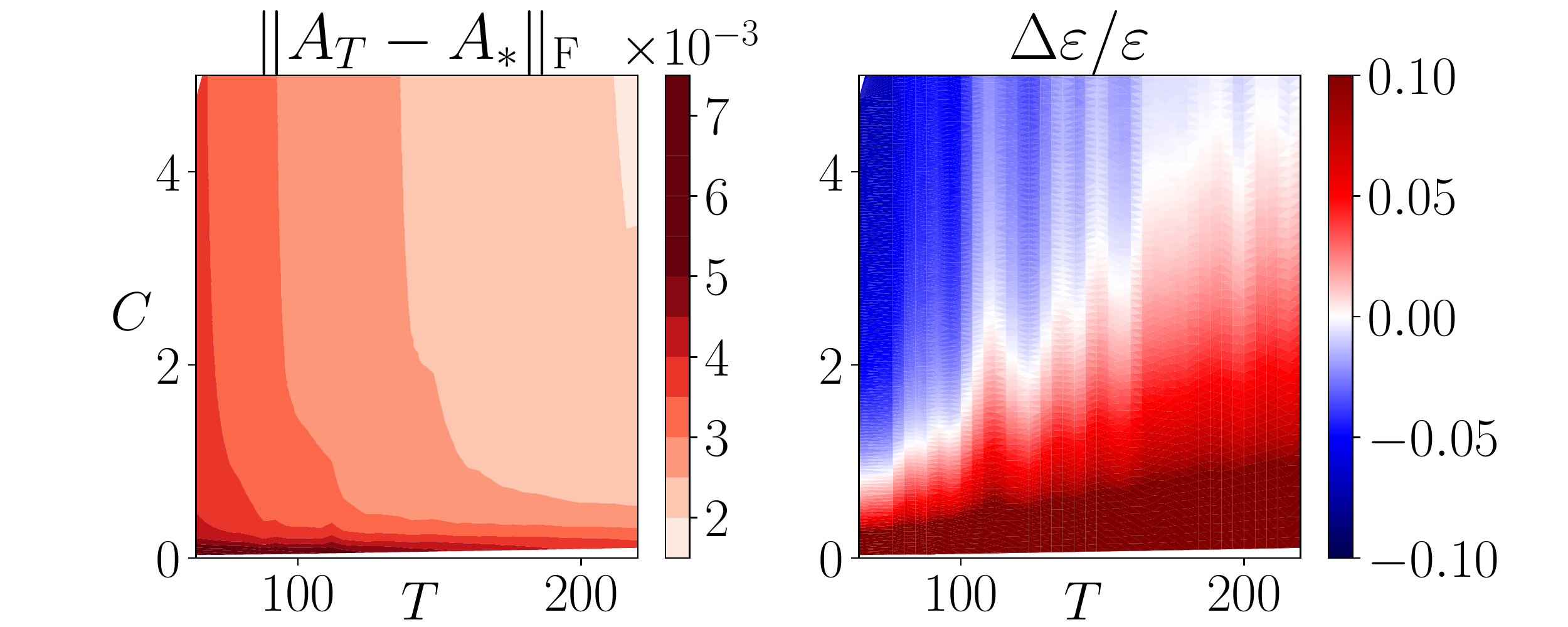}
    \caption{Experimental $(T,C)$ diagram. \textbf{Left} Performance of the gradient algorithm, with varying $T$ and $C$ (varying number of gradient steps). \textbf{Right} Relative performance of the gradient algorithm with respect to the greedy algorithm: negative means that the gradient performs better.}
    \label{fig:diagrams-exp}
\end{figure}

\subsection{Average estimation error}
We now test the performances of our algorithms on random matrices, with the same settings as in the previous experiment. For the gradient algorithm, the minimal number of gradient iterations to reach maximum performance for was found to be $n_{\mathrm{gradient}} = 120$. For each matrix~$A_{\star}$, we also compute an oracle optimal control using Algorithm~\ref{algorithm:gradient_planning} with a batch size of~$b=100$, and run a random input baseline (see~Example~\ref{example:random}), and the TOPLE algorithm of \cite{wagenmaker2021taskoptimal}.
\par
Both the gradient algorithm and the greedy algorithm closely approach the oracle. The former performs slightly better than the latter in average. However, the computational cost of the gradient algorithm is far larger, as Table \ref{table:computational-cost} shows. Indeed, the number of gradient steps to reach convergence in this setting is found to be of order ${n_{\mathrm{gradient}} \simeq 100}$. Note that the number of sub-gradient steps for the TOPLE algorithm is found to be  $n_{\mathrm{TOPLE}} \simeq 1000$, and so~${n_{\mathrm{TOPLE}} \simeq 20 \times  n_{\mathrm{gradient}}}$.
\begin{figure}[h]
    \centering
    \includegraphics[height=7cm]{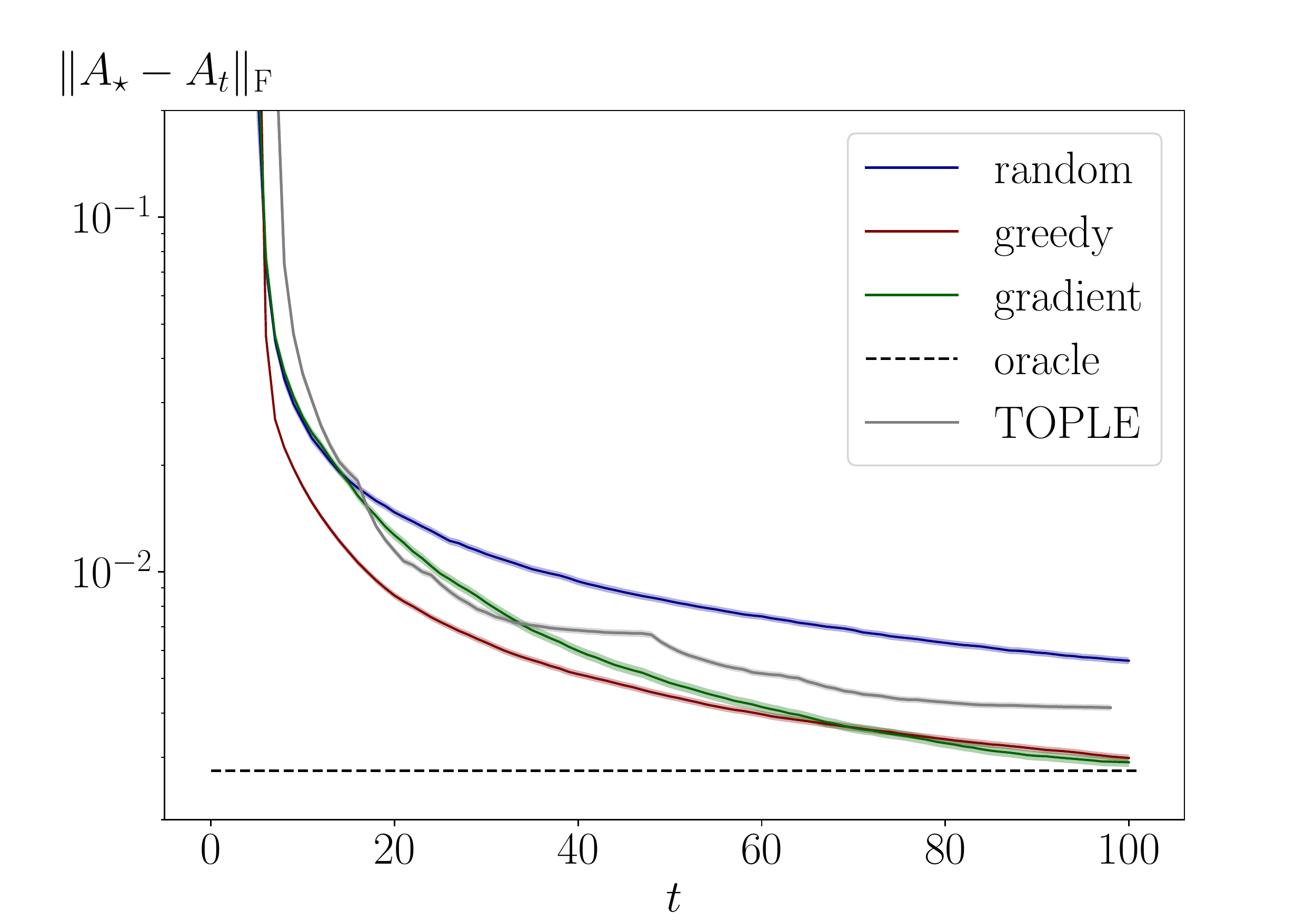}
    \caption{Identification error for random $A_\star$ averaged over 1000 samples. }
    \label{fig:random-A}
\end{figure}

\begin{table}[!t]
    \renewcommand{\arraystretch}{1.3}
    \caption{Average computational rate for the different algorithms.}
    \label{table:computational-cost}
    \centering
    \begin{tabular}{l|llll}
        \bfseries & \bfseries Random
                  & \bfseries TOPLE\cite{Wagenmaker2020}
                  & \bfseries Gradient
                  & \bfseries Greedy
        \\
        \hline
        $c$       & $1$                                  & $n_{\mathrm{TOPLE}}\times 0.02$ & $n_{\mathrm{gradient}}\times 0.5$ & $2.36$
    \end{tabular}
\end{table}

\subsection{Identification of an aircraft system}
We now study a more realistic setting  from the field of aeronautics: we apply system identification to an aircraft system. We use the numerical values issued in a report from the~NASA \cite{gupta1976application}.
The lateral motion of a Lockheed Jet star is described by the slideslip and roll angles and the roll and yaw rates $(\beta, \phi, p, r)\transp{} := x$. The control variables are the aileron and rudder angles $(\delta_{\mathrm{a}}, \delta_{\mathrm{r}} ) := u$.
The linear dynamics for an aircraft flying at 573.7 meters/sec at~6.096 meters are given by the following matrix, obtained after discretization and normalization of the continuous-time system \cite{gupta1976application}:

\begin{equation}
    A_{\star} =
        \begin{pmatrix}
            .955 & -.0113 & 0       & -.0284 \\
            0    & 1      & .0568   & 0      \\
            -.25 & 0      & -.963   & .00496 \\
            .168 & 0      & -.00476 & -.993
        \end{pmatrix},
\end{equation}
\begin{equation}
    B_\star = 0.1 \times
        \begin{pmatrix}
            0    & 0.0116 \\
            0    & 0      \\
            1.62 & .789   \\
            0    & -.87   \\
        \end{pmatrix},
\end{equation}
and $\sigma=1$, $\gamma \simeq 4$ deg.
\begin{table}[!t]
    \renewcommand{\arraystretch}{1.3}
    \caption{Frobenius error for $A_\star$ in the lateral system of the aircraft, $T=150$. Our oracle algorithm reaches an error of $8.0\times 10^{-2}$. The computational time is expressed in an arbitrary unit.}
    \label{table:aircraft}
    \centering
    \begin{tabular}{l|llll}
        \bfseries & \bfseries Random
                  & \bfseries TOPLE \cite{Wagenmaker2020}
                  & \bfseries MSE gradient
                  & \bfseries Greedy
        \\
        \hline
        Error     & $1.1\times10^{-1}$                    & $8.6\times10^{-2}$ & $8.3 \times10^{-2}$ & $8.2\times10^{-2}$
        \\
        Time      & 1                                     & 55.7               & 25                  & 1.13
    \end{tabular}
\end{table}
We apply our algorithms on this LTI system. Our results are summarized in Table \ref{table:aircraft}.
\par
As we can see, the greedy algorithm outperforms the gradient-based algorithms, both in performance and in computational cost. This could be explained by the fact that the signal-to-noise ratio in this system is of order 1, hence the estimation bias in planning is large and it is more effective to plan one-step-ahead than to do planning over large epochs. We obtain similar results for the longitudinal system of a C-8 Buffalo aircraft \cite{gupta1976application}.


\section{Conclusion}

In this work, we explore a setting for linear system identification with hard constraints on the number of interactions with the real system and on the computing resources used for planning and estimation. We introduce a greedy online algorithm requiring minimal computing resources and show empirically that for small values of interactions with the system, it can actually outperform more sophisticated gradient-based methods. Extension of this approach to optimal control for the LQR is an interesting direction of future research.

\section{Matrix calculus}

\begin{proposition}
    \label{proposition:pinverse_differential}
    On a domain where $X$ has linearly independent columns, $X^+$ is a differentiable function of $X$ and
    \begin{equation}
        \ud X^+ = - X^+ \ud X X^+  +  X^+\transp{X^+} \ud \transp{X} ( I - XX^+).
    \end{equation}
\end{proposition}

\begin{proof}
    See \cite{golub1973differentiation}.
\end{proof}
\begin{lemma}
    \label{lemma:det}
    Let $A \in \mathbb{R}^{k \times \ell}$ and $B \in \mathbb{R}^{n\times m}$. Then
    \begin{equation}
        \det (I_{k, m} + AB) = \det(I_{n, \ell} + BA).
    \end{equation}
\end{lemma}

\begin{proposition}
    \label{proposition:det_expansion}
    Let $M \in \mathbb{R}^{d \times d}$ be a nonsingular matrix and $x, y \in \mathbb{R}^d$. Then
    \begin{equation}
        \det (M + x\transp{y}) = \det M \times (1 + \transp{y}M^{-1}x).
    \end{equation}
\end{proposition}

\begin{proof}
    \begin{equation}
        M + x\transp{y} = M(I + M^{-1} x \transp{y})
    \end{equation}
    Apply Lemma \ref{lemma:det}:
    \begin{equation}
        \begin{aligned}
            \det(M + x\transp{y}) & = \det M \times  \det (I_d + M^{-1} x \transp{y})
            \\
                                  & = \det M \times \det (I_1 + \transp{y} M^{-1} x)
            \\
                                  & = \det M \times (1+ \transp{y} M^{-1} x).
        \end{aligned}
    \end{equation}
\end{proof}


\begin{proof}
    See \cite{vrabel2016note}.
\end{proof}

\begin{proposition}
    \label{proposition:logdet_slope}
    Let $0 < A \leq B$ be positive definite matrices of~$\mathbb{R}^{d \times d}$, and $x \in \mathbb{R}^d$. Then
    \begin{equation}
        \label{eq:logdet_slopes}
        \log \det (A+ x \transp{x}) - \log \det A \geq
        \log \det (B+ x \transp{x}) - \log \det B.
    \end{equation}
\end{proposition}

\begin{proof}
    By Proposition \ref{proposition:det_expansion},
    \begin{equation}
        \log \det (A+ x \transp{x}) - \log \det A  = \log (1+ \transp{x}A^{-1}x)
    \end{equation}
    Since $0 < A \leq B$, both matrices are nonsingular and $0 < B^{-1} \leq A^{-1}$. Hence,
    \begin{equation}
        \begin{aligned}
            \log (1+ \transp{x}A^{-1}x) & \geq \log(1 +\transp{x}B^{-1}x)
            \\
                                        & = \log \det (B+ x \transp{x}) - \log \det B
        \end{aligned}
    \end{equation}
\end{proof}
Proposition \ref{proposition:det_expansion} admits the following generalization.
\begin{proposition}
    \label{proposition:det_expansion_sum}
    Let $M \in \mathbb{R}^{d \times d}$ be a nonsingular matrix and let~${x_1, \dots, x_n, y_1, \dots, y \in \mathbb{R}^d}$. Then

    \begin{equation}
        \begin{aligned}
            \det \left(M + \sum\limits_{i=1}^n x_i \transp{y_i}\right)
             & =  \det M
            \\
             & + \sum\limits_{i=1}^n \transp{x_i} \mathrm{adj}\left(M + \sum\limits_{j=1}^{i-1} x_j \transp{y_j}\right)y_i
        \end{aligned}
    \end{equation}
\end{proposition}

\begin{proof}
    See \cite{vrabel2016note}.
\end{proof}

\begin{proposition}
    \label{proposition:inversion_formula}
    Let $M \in \mathbb{R}^{d \times d}$ be a nonsingular matrix and $x, y \in \mathbb{R}^d$. Then $(M + x\transp{y})$ is nonsingular and
    \begin{equation}
        (M + x\transp{y})^{-1} =  (I_d - \frac{1}{1+ \transp{x}M^{-1}y} x \transp{y})M^{-1}
    \end{equation}
\end{proposition}

\begin{corollary}
    \label{corollary:trace_formula}
    Let $M \in \mathbb{R}^{d \times d}$ be a nonsingular matrix and $x, y \in \mathbb{R}^d$. Then
    \begin{equation}
        \mathrm{tr}\left[(M + x\transp{y})^{-1}\right] = \mathrm{tr}[M^{-1}] - \frac{\transp{y}M^{-1}x }{1+ \transp{x}M^{-1}y}
    \end{equation}
\end{corollary}

\bibliographystyle{unsrt}
\bibliography{references}

\end{document}